\documentclass{article}

\usepackage[preprint]{paper_2026}
\usepackage{amsmath}

\usepackage[utf8]{inputenc} % allow utf-8 input
\usepackage[T1]{fontenc}    % use 8-bit T1 fonts
\usepackage{hyperref}       % hyperlinks
\usepackage{url}            % simple URL typesetting
\usepackage{booktabs}       % professional-quality tables
\usepackage{amsfonts}       % blackboard math symbols
\usepackage{nicefrac}       % compact symbols for 1/2, etc.
\usepackage{microtype}      % microtypography
\usepackage{xcolor}         % colors

\usepackage{graphicx}
\usepackage{subcaption}
\usepackage{hyperref}
\usepackage{url}
\usepackage{graphicx}
\usepackage{enumitem}
\usepackage{hyperref}       % hyperlinks
\usepackage{url}            % simple URL typesetting
\usepackage{booktabs}       % professional-quality tables
\usepackage{amsfonts}       % blackboard math symbols
\usepackage{nicefrac}       % compact symbols for 1/2, etc.
\usepackage{microtype}      % microtypography
\usepackage[table]{xcolor}
\usepackage{amsmath}    
\usepackage{amssymb}   
\usepackage{bm}   
\usepackage[most]{tcolorbox}
\usepackage{subcaption}
\usepackage{amsthm}
\usepackage{amsmath}
\usepackage{amssymb}
\usepackage{wrapfig}
\usepackage{xcolor}
\usepackage{tabularx}
\usepackage{multirow}
\usepackage{array}       
\usepackage{makecell}    
\usepackage{color}
\usepackage{float}
\definecolor{myblue}{RGB}{232,242,255}

\theoremstyle{plain}
\newtheorem{theorem}{Theorem}[section]

\theoremstyle{definition}
\newtheorem{definition}[theorem]{Definition}

\theoremstyle{remark}

\title{Breaking the Tokenizer Barrier: On-Policy Distillation across Model Families}
\title{Breaking the Tokenizer Barrier: On-Policy Distillation across Model Families}

\author{%
  \textbf{Yifan Niu$^{1}$}\thanks{Equal contribution.} \quad
  \textbf{Han Xiao$^{1}$}\footnotemark[1] \quad
  \textbf{Dongyi Liu$^{1}$} \quad
  \textbf{Zelong Wang$^{2}$}
  \\
  \textbf{Dihong Gong$^{2}$} \quad
  \textbf{Yasheng Wang$^{2}$}\thanks{Correspondence to: Yasheng Wang (\texttt{asheryswang@tencent.com}) and Jia Li (\texttt{jialee@ust.hk}).} \quad
  \textbf{Jia Li$^{1,3}$}\footnotemark[2]
  \\
  $^{1}$The Hong Kong University of Science and Technology (Guangzhou)
  \\
  $^{2}$Tencent
  \\
  $^{3}$The Hong Kong University of Science and Technology
}

\begin{document}

\maketitle

\begin{abstract}
On-Policy Distillation (OPD) has become a core technique in the post-training of Large Language Models (LLMs) for transferring knowledge from domain experts to student models. However, existing OPD distillation methods require teacher and student models to share the same tokenizer, restricting the applicability of OPD within the model series. Current mainstream practice typically employs Supervised Fine-Tuning (SFT) on teacher-generated responses for cross-tokenizer distillation, which fails to capture the rich knowledge embedded in the teacher's probability distribution. In this work, we enable the standard on-policy distillation method to operate across model families, ensuring that high-fidelity token-level signals can propagate across different tokenizers with precise token-mapping algorithm. Extensive experiments show that cross-tokenizer OPD is significantly more compute-efficient than baselines on various benchmarks. Our results unlock a broader range of teacher–student pairs for OPD, opening up new avenues for adapting and enhancing interactions between LLMs. The code is available at \href{https://github.com/ivanniu/On-Policy-Distill}{https://github.com/ivanniu/On-Policy-Distill}.
\end{abstract}

\section{Introduction}

On-Policy Distillation (OPD)~\cite{zhao2026selfdistilledreasoneronpolicyselfdistillation,yang2026learningteachergeneralizedonpolicy,sdpo2026reinforcementlearningselfdistillation} has become a core technique in the post-training of Large Language Models (LLMs) for transferring knowledge from experts to student models. OPD is known for its computational efficiency and serves as a tool for continual learning. Leading industry models, such as Qwen3~\cite{yang2025qwen3technicalreport}, MiMo~\cite{coreteam2026mimov2flashtechnicalreport}, and GLM-5~\cite{glm5team2026glm5vibecodingagentic}, have all integrated OPD into their post-training stage and reported significant improvements.  Thinking Machines Lab~\cite{lu2025onpolicydistillation} also validated the efficacy of OPD: with a mere fraction of the computational cost, they successfully replicated Qwen3’s OPD recipe, which demonstrates that "on-policy, dense supervision" is an efficient and effective alternative.

In contrast to off-policy distillation~\cite{jiao2020tinybertdistillingbertnatural,wang2020minilmdeepselfattentiondistillation}, which trains a student model using static sequences generated by a teacher model and suffers from exposure bias, OPD trains on the student's rollouts and leverages the teacher’s per-token log-probabilities as a dense reward signal to refine the student's behavior. However, existing OPD methods~\cite{ko2025distillm2contrastiveapproachboosts,kim2026doesselfdistillationsometimesdegrade,xu2025kdrlposttrainingreasoningllms} confine this token-level supervision to same-family distillation, where the teacher and student models share an identical tokenizer. When OPD across model families, the misalignment of vocabulary (differing in size, encoding preferences, and the handling of special characters) creates a significant barrier, as raw log-probabilities cannot be directly transferred. Cross-tokenizer OPD remains a challenge that has yet to be thoroughly explored.

Cross-tokenizer alignment has recently become feasible due to the convergence of industrial tokenization standards. Most contemporary language models rely on subword-level tokenization algorithms~\cite{sennrich2016neuralmachinetranslationrare,kudo2018sentencepiecesimplelanguageindependent}, specifically Byte-Pair Encoding (BPE)~\cite{zouhar2024formalperspectivebytepairencoding}, to partition text into discrete tokens. Although different tokenizers employ vocabularies with distinct encoding characteristics~\cite{Nawrot_2023}, they all share the underlying Unicode/byte-level representation~\cite{xue2022byt5tokenfreefuturepretrained,pagnoni2024bytelatenttransformerpatches}. 
Furthermore, despite the discrepancy in actual token IDs, different tokenizers partition the same underlying text into comparable semantic and character-level fragments, which provides a structural foundation for synchronization. These commonalities effectively lower the barriers between different vocabulary spaces, making cross-tokenizer distillation possible.

In this work, we address the challenge of cross-tokenizer on-policy distillation and propose an adaptive chunk alignment algorithm and credit assignment framework that bridges the vocabulary gap across different model families.  We propose a dual-pointer chunk alignment algorithm to identify the minimal synchronized chunks between the student and teacher sequences. Building upon this alignment, we introduce a chunk-level credit assignment mechanism based on semantic priors. We derive a closed-form solution that scales the student’s log-probabilities, which preserves the student's internal semantic structure while adhering to the guidance of the teacher model. To the best of our knowledge, it is the first work to enable high-fidelity, per-token OPD across different model families. Extensive experiments show that cross-tokenizer OPD is significantly more compute-efficient than off-policy distillation on various benchmarks. Our results unlock a broader range of teacher–student pairs for OPD, opening up new avenues for adapting and enhancing interactions between LLMs.

\section{Related Work}
\subsection{On-Policy Distillation} 
On-Policy Distillation (OPD) optimizes student models by providing dense teacher supervision on self-generated trajectories to have better distribution alignment and performance. In previous works, MiniLLM~\cite{gu2026minillmonpolicydistillationlarge} pioneered on-policy refinement by employing a reverse KL objective to enforce mode-seeking behavior. GKD~\cite{agarwal2024onpolicydistillationlanguagemodels} then established a unified framework by formalizing the sampling distribution as a flexible interpolation between student trajectories and teacher data across diverse divergences. DistiLLM~\cite{ko2024distillmstreamlineddistillationlarge} later stabilized the paradigm by introducing skewed KL divergences, leading to a more tractable optimization process.
Recent research~\cite{song2026surveyonpolicydistillationlarge} classifies feedback signals in OPD into logits-based, outcome-based, and self-play rewards. Logits-based signals~\cite{jung2025toditokenwisedistillationfinegrained,xu2026paceddistillationonpolicyselfdistillation} provide token-level alignment by evaluating student outputs against the teacher's continuous probability distribution. To prevent the computational overhead of full logit, outcome-based methods~\cite{xu2025kdrlposttrainingreasoningllms,zhang2025aligndistiltokenlevellanguagemodel} use sequence-level reward signals. Moreover, self-play methods~\cite{chen2024selfplayfinetuningconvertsweak,zhao2026selfdistilledreasoneronpolicyselfdistillation} allow models to act as both student and teacher, enabling teacher-free alignment without external intervention.
However, these alignments require student and teacher models are the same tokenizer, restricting the applicability of OPD within the model series.

\subsection{Cross-Tokenizer Distillation} Recent research on cross-tokenizer distillation can be divided into two categories: logit-level alignment and text-level alignment. Logit-level alignment methods aim to make the logits of teachers and students in different vocabulary spaces comparable by transforming or reconstructing the prediction distribution. ULD~\cite{boizard2025crosstokenizerdistillationuniversallogit} and MultiLevelOT~\cite{cui2025multileveloptimaltransportuniversal} model cross-tokenizer Logit Alignment as an optimal transport problem, mitigating vocabulary mismatch by transferring probability quality between heterogeneous vocabularys; while CDM~\cite{Li_2024} utilizes entropy-weighted dynamic programming for context-dependent sequence alignment and combines it with dynamic vocabulary mapping to construct a comparable logit space. Text-level alignment methods use text units as a bridge between teachers and students. ALM~\cite{minixhofer2025universalcrosstokenizerdistillationapproximate} directly uses coarse-grained word chunks as the smallest unit to transmit probabilistic information, while BLD~\cite{singh2026crosstokenizerllmdistillationbytelevel} aims for fine-grained byte-level alignment, introducing an additional decoder to reconstruct byte-level probabilistic information. However, existing methods are mostly designed for offline distillation strategies and cannot effectively solve the token-level signal propagation problem in OPD.

\section{Preliminaries} 

\textbf{Notation.}  Let $x = \{x_1, \ldots, x_n\}$ denotes an input query and $y = \{y_1, \ldots, y_m\}$ is the response. The prefix of tokens prior to step $t$ is defined as $y_{<t} = (y_1, \ldots, y_{t-1})$. We consider two language models: a student $\pi_\theta$
and a teacher $\pi_T$, each defining a next-token distribution
$\pi(\cdot \mid x, y_{<t})$ over a student vocabulary $\mathcal{V}_S$ or teacher vocabulary $\mathcal{V}_T$.
We use $y \sim \pi_\theta(\cdot \mid x)$ to represent a response
sampled autoregressively from the student,
and $\mathcal{D}_x =\{x^{(i)}\}_{i=1}^N$ denotes the corresponding query dataset.

\subsection{Tokenization} Tokenization functions fuse bytes into larger (e.g., subword) tokens. Given the text $x$ and the teacher tokenization functions $T(\cdot)$, it encode the text to a set of tokens $T(x) = \{t_1, t_2, \dotsc,t_n\}$. Then we have $t_1\odot t_2 \odot \dotsc \odot t_n = x$, where $\odot$ is the concatenation operator.  Similarly, the student tokenization function $S(\cdot)$ maps the text to $S(x) = \{s_1, s_2, \dotsc, s_m\}$.  Furthermore, the tokenization function is injective, and there exists a left inverse $D$ such that $x = D(T(x)),\forall x$. We call $D$ the \textit{detokenization function}. Importantly, $D$ is only a \textit{left} inverse, i.e., there may exist a token sequence $\bm{t}$ such that $\bm{t} \neq T(D(\bm{t}))$ since $T$ is not necessarily bijective. Generally, the vocabulary does not contain "empty tokens", i.e., $|D(t_i)| > 0$ for any single token $t_i \in \mathcal{V}$. This ensures the strict monotonicity of the detokenization function, i.e., $|D(\bm{t} \odot t')| > |D(\bm{t})|$. For simplicity, we do not distinguish between teacher $D_T$ and student $D_S$ in the following and denote them both as $D$.

\subsection{On-Policy Distillation}
On-Policy Distillation (OPD) computes supervision on trajectories
sampled from the current student $\pi_\theta$. Given a prompt
$x \sim \mathcal{D}_x$, the student samples a response
$\hat{y} = (\hat{y}_1, \ldots, \hat{y}_{T})
  \sim \pi_\theta(\cdot \mid x)$.
Both models are then evaluated on the student-generated prefixes
$\hat{y}_{<t}$, yielding two next-token distributions at each
step~$t$:
$\pi_\theta(v \mid x, \hat{y}_{<t})$ and
$\pi_T(v \mid x, \hat{y}_{<t})$
for $v \in \mathcal{V}$. In this work, we follow the thinking machines lab~\cite{lu2025onpolicydistillation} and adopt the token-level reverse KL as the objective function. This is the lightweight variant of OPD, evaluating only the tokens sampled by the student model; it also represents the most common implementation in previous on-policy distillation works~\cite{lu2025onpolicydistillation,coreteam2026mimov2flashtechnicalreport,yang2025qwen3technicalreport}. We set the per-token advantage to the negative reverse KL, the divergence between the student’s and teacher’s distribution for each token conditioned on the same prior trajectory:
\begin{equation}
\hat{A}_{i,t} = \log \pi_T(\hat{y}_{i,t} \mid x, \hat{y}_{i,<t}) - \log \pi_{\theta_{\mathrm{old}}}(\hat{y}_{i,t} \mid x, \hat{y}_{i,<t}).
\end{equation}
The reward function minimizes the reverse KL, which pushes the student to approximate the teacher’s behavior in every state the student finds itself in. Then we call the RL importance-sampling loss function to perform the training update on the student model:
\begin{equation}\label{eq:grpo}
\small
\begin{aligned}
\mathcal{J}_\text{OPD}& =\ 
\mathbb{E}_{x \sim \mathcal{D},\, y \sim \pi_{\theta_{\mathrm{old}}}} 
  \min\left(
    r_{i,t}(\theta)\, \hat{A}_{i,t},\, 
    \mathrm{clip}\left(
      r_{i,t}(\theta),\, 1{-}\epsilon,\, 1{+}\epsilon
    \right)\, \hat{A}_{i,t}
  \right),
\end{aligned}
\end{equation}
where $r_{i,t}(\theta) = \frac{\pi_\theta(y_{i,t}|x, y_{i,<t})}{\pi_{\theta_\mathrm{old}}(y_{i,t}|x, y_{i,<t})}$ represents the importance-sampling term, and $\epsilon$ is the clipping hyper-parameter.

\section{Method}
In this section, we introduce the method for cross-tokenizer On-Policy Distillation. The challenge lies in the discrepancy between the student vocabulary and the teacher vocabulary, which prevents a direct token-to-token mapping of supervision signals. We first introduce our proposed chunk alignment algorithm to bridge this gap, followed by a chunk-level credit assignment based on semantic prior.

\subsection{Adaptive Chunk Alignment}
\setlength{\columnsep}{5pt}
\begin{wrapfigure}{r}{0.45\linewidth}
\begin{center}
\vspace{-0.6cm}
\includegraphics[width=0.45\textwidth]{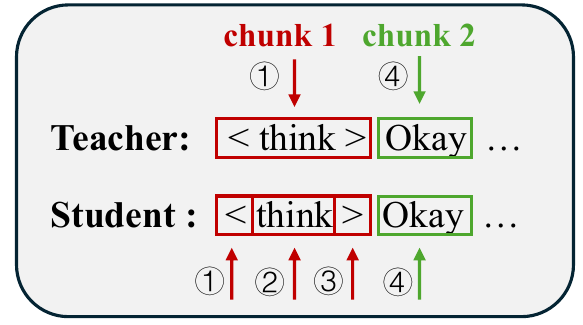}
\end{center}
\vspace{-0.2cm}
\caption{Adaptive Chunk Alignment. }
\vspace{-0.4cm}
\label{HNGFNEXP}
\end{wrapfigure}
Given a student-generated response $y^S = \{s_1, \dots, s_m\} \sim \pi_\theta(\cdot \mid x)$ and its teacher-tokenized counterpart $y^T = \{t_1, \dots, t_k\} = T(D(y^S))$, our goal is to identify synchronization points between the two sequences. Due to differing subword boundaries, a single token in one sequence may correspond to a partial or multiple tokens in the other. We propose the \textbf{Dual-Pointer Chunk Alignment (DPCA)} algorithm, which leverages the deterministic and monotonic properties of the detokenization function $D$ to find the \textit{minimal synchronized chunks}.

\begin{definition}[Synchronization Point]
\label{def:sync_point}
An index pair $(i, j)$ is a \textit{synchronization point} if the decoded prefixes match: $D(y^S_{1:i}) = D(y^T_{1:j})$. 
\end{definition}

\textbf{Algorithm Description.} We maintain two pointers, $p_S$ and $p_T$, representing the current boundaries in the student and teacher sequences, respectively. Starting from any synchronization points $p_S$ and $p_T$ (initialized to $0$), we seek the smallest increments $\Delta m, \Delta k \ge 1$ such that:
\begin{equation}
D({s_{p_S}, \dots, s_{p_S + \Delta m - 1}}) = D({t_{p_T}, \dots, t_{p_T + \Delta k - 1}}).
\end{equation}
To find these increments, we employ a greedy "catch-up" strategy based on the length of the decoded strings. Within a chunk, let $i$ be the current offset for the student token and $j$ be the offset for the teacher token. We iteratively advance the pointers as follows:
\begin{itemize}\item Let $\text{str}_S = D(\{s_{p_S}, \dots, s_{p_S + i - 1}\})$ and $\text{str}_T = D(\{t_{p_T}, \dots, t_{p_T + j - 1}\})$.
\item \textbf{If} $|\text{str}_S| < |\text{str}_T|$, we increment $i$ to extend the student's fragment.
\item \textbf{Else if} $|\text{str}_S| > |\text{str}_T|$, we increment $j$ to extend the teacher's fragment.
\end{itemize}
The process for a single chunk terminates at the first instance where $\text{str}_S = \text{str}_T$. Upon matching, we define the aligned pair as a \textit{synchronized chunk}, update the base pointers $p_S \leftarrow p_S + i$ and $p_T \leftarrow p_T + j$, and repeat the process until both sequences are exhausted.

\textbf{Theoretical Analysis of DPCA.}
To ensure high-fidelity knowledge transfer, the alignment mechanism must operate at the highest possible resolution. Any alignment that aggregates multiple independent synchronization points into a single block will obscure the dense supervision signals provided by the teacher. By leveraging the strict monotonicity of the detokenization function $D$, the DPCA algorithm effectively ensures the completeness and minimality of identified chunks through asynchronous greedy search. In other words, DPCA ensures the highest partition resolution. Detailed proof are provided in Appendix~\ref{app:proof}.

\begin{theorem}[Completeness and Minimality of DPCA]
\label{thm:minimal_partition}
The DPCA algorithm ensures: 1.\textbf{Completeness}: It identifies the entire set of synchronization points between $y^S$ and $y^T$. 2.\textbf{Minimality}: The resulting partitions are minimal synchronized chunks, such that no chunk can be further subdivided into smaller synchronized sub-sequences.
\end{theorem}

The \textit{completeness} guarantees that the algorithm never skips a synchronization point.  The \textit{minimality} implies that each identified chunk is a minimal closure. Consequently, the resulting partition has found all possible minimal chunks. This is crucial for OPD as it ensures that the teacher's log-probabilities are localized to the smallest possible token clusters, which avoids the ambiguous signal over excessively large chunks.

\subsection{Chunk-level Credit Assignment via Semantic Prior}

Once the sequences are partitioned into a set of minimal synchronized chunks $\mathcal{P} = \{ (\bm{s}_c, \bm{t}_c) \}_{c=1}^K$, the teacher's supervision signals should be projected back onto the student tokens. The projection, also referred to as credit assignment, must satisfy the total probability constraint while respecting the student's internal semantic structure.

\textbf{Joint Probability Matching.} Since each synchronized chunk $(\bm{s}_c, \bm{t}_c)$ represents the same content $x_{chunk} = D(\bm{s}_c) = D(\bm{t}_c)$, we assume that the joint probability of generating this content should be equivalent under both the student and teacher distributions:
\begin{equation}
P(\bm{s}_c \mid x, y^S_{<c}) = P(\bm{t}_c \mid x, y^T_{<c}),
\end{equation}
where $y^S_{<c}$ and $y^T_{<c}$ are the prefixes to the current chunk. Under the autoregressive factorization, this equality implies a product of conditional probabilities:
\begin{equation}\prod_{i=1}^{|\bm{s}_c|} \pi_{\theta}(s_{c,i} \mid x, y^S_{<c} \odot s_{c,<i}) = \prod_{j=1}^{|\bm{t}_c|} \pi_T(t_{c,j} \mid x, y^T_{<c} \odot t_{c,<j}).
\end{equation}
Applying the logarithm transform, we obtain the constraint that the sum of log-probabilities within a synchronized chunk should be equal:
\begin{equation}
\label{eq:sum_constraint}
\sum_{i=1}^{|\bm{s}_c|} \log \pi_{\theta}(s_{c,i} \mid \dots) = \sum_{j=1}^{|\bm{t}_c|} \log \pi_T(t_{c,j} \mid \dots).
\end{equation}

\textbf{Optimization with Semantic Prior.} In cross-tokenizer OPD, the teacher's density is aggregated over the teacher tokens $\bm{t}_c$. Let $\mathcal{L}_T^{(c)} = \sum_{j} \log \pi_T(t_{c,j} \mid \dots)$ be the target log-likelihood for the $c$-th chunk. To distribute this target to student tokens $\bm{s}_c = \{s_1, \dots, s_m\}$ while preserving the student's semantic prior $\bm{p} = \{ \pi_{\theta_{\text{old}}}(s_i \mid \dots) \}_{i=1}^m$ (i.e., its pre-existing internal distribution), we seek a target distribution $\bm{q}$ that maintains the maximum directional alignment with the prior $\bm{p}$:
\begin{equation}
\begin{aligned}
\min_{\bm{q}} \quad & \frac{1}{2} \sum_{i=1}^m \left( \log q_i - k \log p_i \right)^2 \quad
\text{s.t.} \ & \sum_{i=1}^m \log q_i = \mathcal{L}_T^{(c)}.
\end{aligned}
\end{equation}
By solving this constrained optimization problem with Lagrange multipliers, we obtain a closed-form solution where the target log-probabilities $\bm{q}$ are distributed proportionally to the student's prior:
\begin{equation}
\log q_i = \frac{\mathcal{L}_T^{(c)}}{\mathcal{L}_S^{(c)}}  \log p_i .
\end{equation}
where $\mathcal{L}_S^{(c)} = \sum_{j=1}^m \log p_j$ is the student's total log-likelihood for the chunk. The target $\log q_i$ is a linearly scaled version of the prior $\log p_i$, where the scaling factor is determined by the ratio of the chunk-level log-probabilities.

\textbf{Advantage Computation.} With the target log-probability, the per-token advantage is calculated as:
\begin{equation}\label{esadv}
\hat{A}_{i,t} = \log q_i - \log p_i = \left( \frac{\mathcal{L}_T^{(c)}}{\mathcal{L}_S^{(c)}} - 1 \right) \log p_i,
\end{equation}
\textbf{Remark.} This proportional assignment ensures that the relative semantic structure within the chunk is preserved. For instance, tokens that the student initially found more "surprising" (i.e., larger negative log-probabilities) receive a proportionally larger share of the total advantage. Furthermore, this formulation is self-consistent: in the case of 1:1 alignment ($m=1$), the advantage naturally simplifies to $\hat{A}_{i,t} = \mathcal{L}_T^{(c)} - \mathcal{L}_S^{(c)}$, recovering the standard per-token distillation objective. 

\subsection{Cross-Tokenizer On-Policy Distillation}

With the token-level advantages $\hat{A}_{i,t}$ established via chunk-level credit assignment, we integrate them into an on-policy framework to update the student model. Our objective function follow the standard form of OPD, and call the RL importance-sampling loss function:
\begin{equation}
\small
\begin{aligned}
\mathcal{J}_\text{OPD}& =\ 
\mathbb{E}_{x \sim \mathcal{D},\, y \sim \pi_{\theta_{\mathrm{old}}}} 
  \min\left(
    r_{i,t}(\theta)\, \hat{A}_{i,t},\, 
    \mathrm{clip}\left(
      r_{i,t}(\theta),\, 1{-}\epsilon,\, 1{+}\epsilon
    \right)\, \hat{A}_{i,t}
  \right),
\end{aligned}
\end{equation}
where $r_{i,t}(\theta) = \frac{\pi_\theta(y_{i,t}|x, y_{i,<t})}{\pi_{\theta_\mathrm{old}}(y_{i,t}|x, y_{i,<t})}$ represents the importance-sampling term, $\epsilon$ is the clipping hyper-parameter, and  $\hat{A}_{i,t}$ is the per-token advantage assigned via the semantic prior (Eq. \ref{esadv}). The ovrall objective ensures that the student model is updated in such a way that, even if its tokenization differs from that of the teacher model, its internal semantic structure remains consistent with the supervision of the teacher model. Training on the student model's own rollouts, the model is able to learn from and correct its unique distributional errors within its own generative context.

\section{Experiments}
\label{sec:experiments}

% In this section, we designed experiments to demonstrate our method and prove its effectiveness.

\subsection{Experimental Setup}
\label{sec:exp_setup}

\textbf{Settings.} To evaluate the performance of cross-tokenizer OPD, we considered two settings. \textbf{(1)} Qwen3-8B~\cite{yang2025qwen3technicalreport} and Llama-3.1-8B-SFT~\cite{grattafiori2024llama3herdmodels}, where Qwen3-8B acts as the teacher and Llama-3.1-8B-SFT as the student, test the effectiveness of cross-tokenizer OPD across different model families and tokenizers. \textbf{(2)} DeepSeek-R1~\cite{Guo_2025} and DeepSeek-R1-Distill-Qwen-7B, where DeepSeek-R1 acts as the teacher and DeepSeek-R1-Distill-Qwen-7B as the student, examine whether cross-tokenizer OPD can still effectively transmit token‑level signals from powerful teacher model to student model.
In all settings, the student first obtains an initial model through supervised fine-tuning, and then further trains it on on-policy trajectories sampled by the student itself using the OPD method proposed in this paper. Specially, in SFT stage, we train student models on 400K samples from OpenThoughts by LLaMA-Factory~\cite{zheng2024llamafactoryunifiedefficientfinetuning}.
In cross-tokenizer OPD stage, we implemented training process based on VeRL~\cite{sheng2024hybridflow} and deployed a separate teacher service to score the on-policy trajectories generated by the student online. More hyperparameter information is provided in the Appendix~\ref{appendix:details}.

\textbf{Training data.}
The training data primarily comes from OpenThoughts~\cite{guha2025openthoughtsdatarecipesreasoning} and DeepMath~\cite{he2025deepmath103klargescalechallengingdecontaminated}. OpenThoughts covers various tasks (\emph{e.g.}, math, code,  science question, and general reasoning), which includes high quality data from 26 sources and uses QwQ-32B~\cite{qwq32b} for all annotations.
DeepMath provides more challenging, verifiable, and decontaminated math problems with rich data format and data diversity. We selected 400K samples from OpenThoughts to initialize the student's SFT, and respectively extracted 20K prompts from OpenThoughts and DeepMath for OPD training.

% off policy 分层分点
\textbf{Baselines.}
To demonstrate the effectiveness of our method, we compare it against: (1) traditional method (SFT); (2) previous cross-tokenizer methods, (\emph{e.g.}, ALM~\cite{minixhofer2025universalcrosstokenizerdistillationapproximate} and CDM~\cite{Li_2024}). ALM uses surface-form blocks as bridging units between the teacher and student, passing teacher likelihoods to the student at the block level, while CDM employs a context-dependent dynamic mapping strategy to construct finer-grained likelihood matching between heterogeneous tokenizers. Unlike these off-policy or static matching methods, our OPD is trained directly on the on-policy trajectories generated by the student, utilizing the teacher's dense token-level signals. 

% We compare three training methods: SFT, SFT+ALM~\cite{minixhofer2025universalcrosstokenizerdistillationapproximate}, and SFT+CDM~\cite{Li_2024}. SFT uses OpenThoughts data for supervised fine-tuning of the student, serving as a common initialization for all subsequent methods. SFT+ALM uses surface-form blocks as bridging units between the teacher and student, passing teacher likelihoods to the student at the block level. SFT+CDM further employs a context-dependent dynamic mapping strategy to construct finer-grained likelihood matching between heterogeneous tokenizers.

\textbf{Benchmarks and evaluation.} We evaluated the model's performance on AIME24, AIME25, AIME26~\cite{aime26}, MATH-500~\cite{hendrycksmath2021}, GPQA-Diamond~\cite{rein2023gpqagraduatelevelgoogleproofqa}, and LiveCodeBench~\cite{jain2024livecodebench}. AIME and MATH-500 were used to measure mathematical reasoning ability, GPQA-Diamond to measure scientific question answering and complex knowledge reasoning ability, and LiveCodeBench to measure code generation and algorithmic reasoning ability. For AIME and MATH-500, we used answer extraction and exact-match judgment; for GPQA-Diamond, we reported multiple-choice accuracy; and for LiveCodeBench, we reported the test case accuracy. More details are in the Appendix~\ref{app:data}.

% 放setting
% \textbf{Implementation details.}
% The experiment consisted of two phases. The first phase involved SFT initialization of the student tokens, where we trained the student checkpoint using LLaMA-Factory~\cite{zheng2024llamafactoryunifiedefficientfinetuning} on 400K OpenThoughts samples. The second phase involved cross-tokenizer distillation. We implemented an OPD training process based on VeRL~\cite{sheng2024hybridflow} and deployed a separate teacher service to score the on-policy trajectories generated by the student online. During training, the student was responsible for sampling and parameter updates, while the teacher only provided log-probability supervision. For cases where the teacher and student tokenizers were inconsistent, we performed text-level chunk alignment within each sequence and projected the teacher's chunk-level probability budget onto the student token sequence.

\subsection{Main Results}
\label{sec:main_results}

% Overall, SFT provides a strong initialization for the student models, while ALM and CDM further improve performance by matching teacher likelihoods across heterogeneous tokenizers. However, these methods remain off-policy which cannot provide dense teacher supervision on the student's own sampled outputs. That's why OPD consistently achieves the best average performance in both teacher--student settings.

\begin{table*}[!t]
\centering
\caption{Main results on reasoning benchmarks. The best result within each block is highlighted in \textbf{bold}, and the second-best result is \underline{underlined}.}
\footnotesize
\setlength{\tabcolsep}{5pt}
\renewcommand{\arraystretch}{1.03}
\begin{tabular}{l|cccc|cc|c}
\toprule[1.3pt]
\textbf{Model} & \multicolumn{4}{c|}{\textbf{Math}} & \multicolumn{2}{c|}{\textbf{Science / Code}} & \textbf{Avg.} \\
\cmidrule(lr){2-5}\cmidrule(lr){6-7}\cmidrule(lr){8-8}
& \textbf{AIME24} & \textbf{AIME25} & \textbf{AIME26} & \textbf{MATH} & \textbf{GPQA-D} & \textbf{LCB} & \textbf{Avg.} \\
\midrule
\multicolumn{8}{l}{\textbf{Qwen3-8B $\rightarrow$ Llama3.1-8B-SFT}} \\
\midrule
+ SFT & 38.5 & 31.7 & 36.5 & 80.8 & 1.5 & 21.8 & 35.1 \\
+ SFT + ALM & 39.0 & \underline{32.5} & 37.1 & 81.2 & 2.0 & 22.4 & 35.7 \\
+ SFT + CDM & \underline{39.2} & 32.3 & \underline{37.3} & \underline{81.4} & \underline{2.5} & \underline{22.9} & \underline{35.9} \\
\rowcolor{myblue}+ \textbf{SFT + OPD} & \textbf{44.4} & \textbf{43.3} & \textbf{41.7} & \textbf{82.6} & \textbf{7.6} & \textbf{24.8} & \textbf{40.7} \\
\midrule
\multicolumn{8}{l}{\textbf{DeepSeek-R1 $\rightarrow$ DeepSeek-R1-Distill-Qwen-7B}} \\
\midrule
+ SFT & 53.3 & 34.2 & 46.7 & 87.4 & 46.9 & 24.3 & 48.8 \\
+ SFT + ALM & 53.8 & 35.0 & 47.3 & \underline{88.0} & \underline{47.6} & 24.0 & 49.3 \\
+ SFT + CDM & \underline{54.0} & \underline{35.6} & \underline{48.0} & 87.8 & 47.4 & \underline{24.7} & \underline{49.6} \\
\rowcolor{myblue}+ \textbf{SFT + OPD} & \textbf{56.7} & \textbf{42.5} & \textbf{50.8} & \textbf{89.0} & \textbf{48.1} & \textbf{25.1} & \textbf{52.0} \\
\bottomrule
\end{tabular}
\label{tab:main_results}
\end{table*}

Table~\ref{tab:main_results} presents the main results across both cross-tokenizer distillation settings. In the Qwen3-8B $\rightarrow$ Llama3.1-8B-SFT setting, the SFT-initialized student attains an average of 35.1. ALM and CDM yield modest improvements (35.7 and 35.9), indicating that static cross-tokenizer likelihood alignment provides limited gains. OPD raises the average to 40.7, surpassing SFT, ALM, and CDM by 5.6, 5.0, and 4.8 points respectively. The gains are consistent across all benchmarks, with particularly notable improvements on AIME25 (+11.6) and GPQA-Diamond (+6.1), confirming that dense on-policy teacher feedback delivers a stronger supervisory signal than static off-policy alignment.

In the DeepSeek-R1 $\rightarrow$ DeepSeek-R1-Distill-Qwen-7B setting, the student starts from a stronger initialization (48.8 average). ALM and CDM again provide marginal improvements (49.3 and 49.6), while OPD reaches 52.0, outperforming CDM by 2.4 points on average and by 2.7-6.9 points on individual AIME benchmarks. The consistent improvement across MATH-500, GPQA-Diamond, and LiveCodeBench demonstrates that the method generalizes beyond a single teacher--student pair: even with a well-initialized student, on-policy token-level supervision produces substantial gains.

Overall, OPD achieved the best performance across all benchmarks under both settings. The key distinction between OPD and ALM or CDM lies in the fact that OPD employs a teacher strategy to evaluate trajectories sampled from a student strategy, rather than having the student train against a fixed sequence generated by the teacher. These results confirm cross-tokenizer distillation can indeed effectively transmit dense teacher supervision signals with a chunk-based alignment mechanism and a log probability-preserving attribution mechanism, thereby achieving superior distillation results.

\subsection{Compute Cost Analysis}
\label{sec:compute_cost}

To evaluate the computational efficiency of OPD, we separately calculated the amount of training data and computational FLOPS required for OPD and SFT to achieve the same AIME accuracy. The amount of training data is the actual number of data points consumed. The FLOPS for SFT includes the floating-point computations required for forward/backward and teacher rollout. The FLOPS for OPD includes the floating-point computations required for student rollout, teacher prefill, and student forward/backward. Throughout the computation process, the student model was Llama3.1-8B-SFT, the teacher model for SFT and OPD was Qwen3-8B.
\begin{wrapfigure}{r}{0.60\textwidth}
    \vspace{-0.1em}
    \centering
    \includegraphics[width=\linewidth]{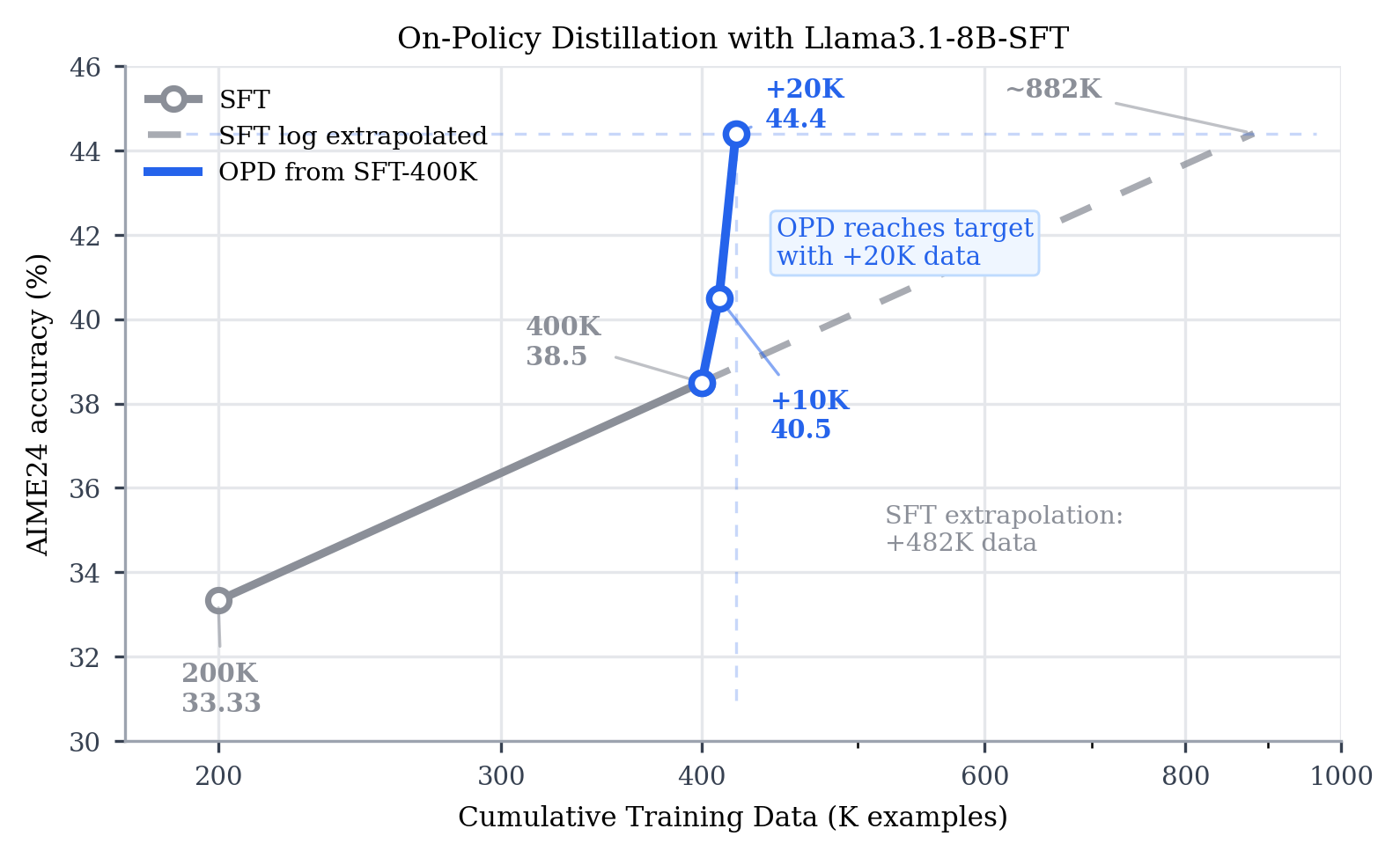}
    \caption{AIME24 compute frontier when distilling Llama3.1-8B-SFT using Qwen3-8B.}
    \label{fig:compute_frontier}
    \vspace{-1.0em}
\end{wrapfigure}
As shown in the figure, the gray line depicts the performance trend of checkpointing as the amount of SFT data increases. It was drawn using checkpoint performance data at 200K and 400K and based on empirically observed log-linearity~\cite{lu2025onpolicydistillation}. The extended dashed line is an extrapolation of this law. The blue curve shows the performance trend of checkpointing as the number of OPD training samples increases. We selected the checkpoint performance data of the model trained with OPD after 10K and 20K samples to plot this graph. Starting from SFT-400K, OPD improved the model's AIME24 accuracy to 40.5\% and 44.4\% with only 10K and 20K samples, respectively. This far exceeds the improvement of SFT with the same amount of data. Based on the log-linearity law, SFT would need approximately 480K more data to achieve the same performance.

Table ~\ref{tab:compute_cost} shows the floating-point computation counts estimated based on the amount of data and the computation path for each training method. As you can see, to achieve the same 44.4\% AIME24 accuracy from an SFT-400K initialized model, continuing with SFT requires approximately 482K additional training data points. The corresponding teacher rollout computation is approximately $6.0 \times 10^{20}$ FLOPs, the student training computation is approximately $5.7 \times 10^{20}$ FLOPs, and the total computation is approximately $1.17 \times 10^{21}$ FLOPs. In comparison, OPD requires only an additional 20K training samples, with teacher prefill computation costing approximately $2.5 \times 10^{19}$ FLOPs, student rollout and training computation costing approximately $2.3 \times 10^{19}$ FLOPs, and total computation costing approximately $4.8 \times 10^{19}$ FLOPs.

As can be seen, to achieve the same AIME24 performance, OPD uses only about 4.1\% of the training data and reduces the total training computation cost to approximately 4\% of the SFT extrapolation scheme. This significant difference demonstrates the substantial advantage of OPD. OPD provides dense teacher supervision signals based on the student's own sampling trajectory, unlike the SFT-based approach which forcibly aligns student responses with the teacher's generated thought processes. This ensures that students can freely explore in their most suitable direction while also guaranteeing the correctness of the optimization direction.

\begin{table}[t]
\centering
\caption{The estimated OPD and SFT computation FLOPs under the Qwen3-8B $\rightarrow$ Llama3.1-8B-SFT configuration. These FLOPs estimates are based on the model configurations of the teacher (Qwen3-8B) and student (Llama3.1-8B) models, assuming each sample contains 16K tokens.}
\vspace{0.45em}
\small
\setlength{\tabcolsep}{5pt}
\renewcommand{\arraystretch}{1.12}
\resizebox{\linewidth}{!}{
\begin{tabular}{lccccc}
\toprule
\textbf{Method} & \textbf{AIME24} & \textbf{Data} & \textbf{Teacher FLOPs} & \textbf{Student FLOPs} & \textbf{CE} \\
\midrule
Initialization: SFT-400K & 38.5\% & -- & -- & -- & -- \\
SFT extrapolated & 44.4\% & +482K & $6.0{\times}10^{20}$ & $5.7{\times}10^{20}$ & 1.0$\times$ \\
\textbf{SFT + OPD} & \textbf{44.4\%} & \textbf{+20K} & $2.5{\times}10^{19}$ & $2.3{\times}10^{19}$ & \textbf{24.1$\times$} \\
\bottomrule
\end{tabular}
}
\vspace{10pt}
\label{tab:compute_cost}
\end{table}

% Table~\ref{tab:compute_cost} reports the corresponding FLOPs estimate. For the SFT extrapolation baseline, we count the teacher-side generation/scoring cost required to obtain additional supervised data and the student-side training cost on that data. For OPD, we count teacher scoring on student rollouts together with the student rollout and update cost. Under this accounting, OPD costs approximately 3.3 $\times$ 10$^{19}$ FLOPs, whereas the SFT extrapolation costs approximately 8.0 $\times$ 10$^{20}$ FLOPs. Thus, for the same 44.4\% AIME24 target, cross-tokenizer OPD provides a 24.1$\times$ compute-efficiency improvement, showing that the proposed projection mechanism preserves the practical efficiency advantage of OPD even when teacher and student tokenizers differ.

\subsection{Case Study}
\label{sec:case_study}

We provide a qualitative example illustrating the process of assigning across tokenizers. Figure ~\ref{fig:chunk_alignment} shows the different tokenization results of Qwen and Llama for the same mathematical expression. Gray dashed boxes represent directly aligned (1:1) chunks, where the teacher's log probability can be applied directly; blue boxes represent mismatched chunks, requiring chunk-level reassignment.

This example demonstrates two directions of mismatch in a single response. For the substring \texttt{120}, Qwen generates three tokens (\texttt{1}, \texttt{2}, \texttt{0}), while Llama retains it as a single token; conversely, for \texttt{...]}, Qwen retains it as a single token, while Llama splits it into \texttt{...} and \texttt{]}. This bidirectional mismatch demonstrates that no fixed merge or split heuristic is sufficient to solve the problem, so the algorithm must first identify the synchronized word chunks and then allocate the teacher's log probability budget to each chunk.

As shown in the Figure ~\ref{fig:chunk_alignment}, for the mismatched word chunk \texttt{120}, the log probabilities $-1.81$, $-1.26$, and $-4.04$ of the three words on the Qwen side are merged into a single chunk budget of $-7.12$ and projected onto Llama's single \texttt{120} word; since the log probability of this word on the student side is $-9.21$, a positive advantage $2.09$ is obtained. Conversely, at \texttt{...]}, the budget of a single word on the Qwen side, $-15.91$, is redistributed to the two words in Llama's \texttt{...} and \texttt{]} based on the log probabilities of Llama's two words, resulting in $-7.56$ and $-8.35$ respectively, corresponding to advantages of $-0.63$ and $-0.70$.

\begin{figure}[t]
    \centering
    \makebox[\linewidth][c]{\hspace*{-0.04\linewidth}\includegraphics[width=0.97\linewidth]{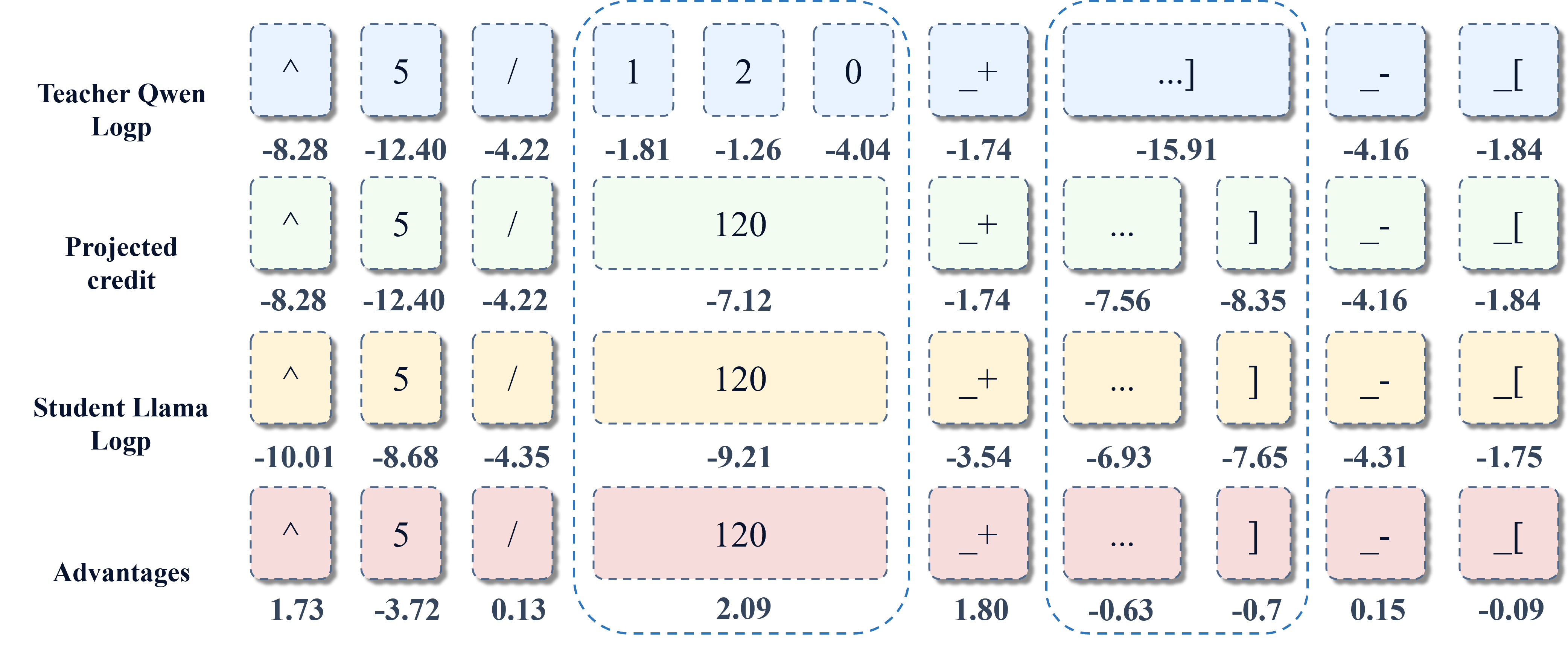}}
    \caption{Case study of cross-tokenizer credit assignment. From top to bottom, the rows show teacher log-probabilities, projected teacher log-probabilities, student log-probabilities, and the resulting advantages computed as projected teacher log-probabilities minus student log-probabilities. Blue groups are mismatched chunks; gray dashed groups are directly aligned chunks.}
    \label{fig:chunk_alignment}
\end{figure}

This demonstrates that this matching-projection process can handle both the merging of teacher multi-word units into student word units and the splitting of teacher word units into student multi-word units, ultimately generating a reasonable advantage signal for each student word unit.

\subsection{Distillation Gain: Same-Tokenizer vs. Cross-Tokenizer OPD}
\label{sec:same_vs_cross}

Finally, we compare the OPD gain under the same-tokenizer and cross-tokenizer settings to isolate the effect of tokenizer projection. We measure the improvement beyond each model's own SFT initialization rather than comparing absolute scores. This controlled comparison reveals whether the projection step decreases the incremental training signal of OPD.

As shown in Table~\ref{tab:same_vs_cross}, cross-tokenizer OPD (Qwen3 $\rightarrow$ Llama3.1-8B-SFT) achieves an average gain of 5.6, comparable to the 5.5 gain of same-tokenizer OPD (Qwen3 $\rightarrow$ Qwen3). The improvement is consistent across benchmark families: cross-tokenizer OPD gains +11.6 on AIME25 (substantially exceeding the same-tokenizer counterpart of +0.9), +5.9 on AIME24, +5.2 on AIME26, and positive gains on MATH (+1.8), GPQA-Diamond (+6.1), and LiveCodeBench (+3.0).

\begin{table}[htbp]
    \centering
    \caption{OPD gain comparison between same-tokenizer and cross-tokenizer settings. Gain is computed as OPD minus the corresponding SFT initialization.}
    \begin{tabular*}{\linewidth}{l@{\extracolsep{\fill}}ccc|ccc}
    \toprule[1.3pt]
    \textbf{Benchmark} & \multicolumn{3}{c|}{\textbf{Same Tokenizer}} & \multicolumn{3}{c}{\textbf{Cross Tokenizer}} \\
    \cmidrule(lr){2-4}\cmidrule(lr){5-7}
    & \textbf{SFT} & \textbf{OPD} & \textbf{Gain} & \textbf{SFT} & \textbf{OPD} & \textbf{Gain} \\
    \midrule
    AIME24 & 59.2 & 66.7 & \textbf{+7.5} & 38.5 & 44.4 & +5.9 \\
    AIME25 & 50.8 & 51.7 & +0.9 & 31.7 & 43.3 & \textbf{+11.6} \\
    AIME26 & 50.0 & 59.2 & \textbf{+9.2} & 36.5 & 41.7 & +5.2 \\
    MATH & 88.9 & 90.0 & +1.1 & 80.8 & 82.6 & \textbf{+1.8} \\
    GPQA-D & 37.1 & 48.4 & \textbf{+11.3} & 1.5 & 7.6 & +6.1 \\
    LCB & 29.3 & 32.3 & +3.0 & 21.8 & 24.8 & +3.0 \\
    \midrule
    Avg. & 52.6 & 58.1 & +5.5 & 35.1 & 40.7 & \textbf{+5.6} \\
    \bottomrule
    \end{tabular*}
    \label{tab:same_vs_cross}
\end{table}

These results demonstrate that tokenizer mismatch does not weaken OPD effectiveness. The comparable relative gain confirms that our alignment and projection procedure preserve the dense supervision signal without modifying either model's tokenizer or introducing additional vocabulary heads. Therefore, a shared vocabulary may not a prerequisite for effective OPD when teacher feedback can be projected onto student tokens in a log probability-conserving manner.

\section{Conclusion}

In this work, we present a cross-tokenizer On-Policy Distillation (OPD) framework that  allows OPD across different model families. We propose an adaptive dual-pointer chunk alignment algorithm and a semantic-prior-based credit assignment mechanism to enable the precise propagation of high-fidelity, token-level signals across different vocabularies. Our results demonstrate that cross-tokenizer OPD is significantly more compute-efficient and effective at capturing dense distributional knowledge. Our work expands the potential for teacher–student pairings and paves the way for more flexible knowledge transfer across diverse model families.

\textbf{Limitations and Broader Impacts}. While our method demonstrates significant improvements across various tasks, due to computational constraints, our experiments are limited to models under 8B parameters; validation on larger‑scale models (\emph{e.g.}, 70B) remains as future work. Our cross‑tokenizer OPD reduces the dependency on shared tokenizers, enabling more flexible teacher–student pairings.

\bibliography{reference}
\bibliographystyle{plain}

        %%%%%%%%%%%%%%%%%%%%%%%%%%%%%%%%%%%%%%%%%%%%%%%%%%%%%%%%%%%%%%%%%%%%%%%%%%%%%%%
        %%%%%%%%%%%%%%%%%%%%%%%%%%%%%%%%%%%%%%%%%%%%%%%%%%%%%%%%%%%%%%%%%%%%%%%%%%%%%%%
        % APPENDIX
        %%%%%%%%%%%%%%%%%%%%%%%%%%%%%%%%%%%%%%%%%%%%%%%%%%%%%%%%%%%%%%%%%%%%%%%%%%%%%%%
        %%%%%%%%%%%%%%%%%%%%%%%%%%%%%%%%%%%%%%%%%%%%%%%%%%%%%%%%%%%%%%%%%%%%%%%%%%%%%%%
        \newpage
        \appendix
        \onecolumn
\section{Details of Experimental Setups}\label{appendix:details}

        \begin{table*}[htbp]
        \centering
        \caption{Main training hyperparameters used for SFT initialization and subsequent OPD training.}
        \small
        \setlength{\tabcolsep}{6pt}
        \renewcommand{\arraystretch}{1.15}
        \begin{tabular}{p{0.27\textwidth}|p{0.31\textwidth}|p{0.31\textwidth}}
        \toprule[1.3pt]
        \textbf{Hyperparameter} & \textbf{SFT} & \textbf{OPD} \\
        \midrule
        Training framework & LLaMA-Factory & VeRL \\
        \midrule
        Teacher--student pairs & -- & Qwen3-8B $\rightarrow$ Llama3.1-8B-SFT; DeepSeek-R1 $\rightarrow$ DeepSeek-R1-Distill-Qwen-7B \\
        \midrule
        SFT base models & Meta-Llama-3.1-8B & corresponding SFT checkpoint \\
        \midrule
        Training data & OpenThoughts & DeepMath and OpenThoughts prompts \\
        \midrule
        Global batch size & 128 examples & 128 prompts \\
        \midrule
        Maximum sequence length & 16,384 tokens & 1,024 prompt tok ens + 16,384 response tokens \\
        \midrule
        Learning rate & $1.0 \times 10^{-4}$ & $1.0 \times 10^{-6}$ \\
        \midrule
        Optimizer & AdamW & AdamW \\
        \midrule
        Gradient clipping & 1.0 & 1.0 \\
        \midrule
        Precision & bf16 & bf16 \\
        \midrule
        Rollout sampling & -- & $T=1.0$, top-$p=1.0$, top-$k=-1$ \\
        \bottomrule
        \end{tabular}
        \label{tab:hyperparams}
        \end{table*}

\section{Details of Dataset and Baseline}\label{app:data}
The following are the details of the datasets used in the benchmark:
\begin{itemize}
     \item \textbf{MATH}~\cite{hendrycksmath2021} consists of 12,500 problems from high school math competitions, measuring the problem-solving ability of machine learning models.
    \item \textbf{LiveCodeBench}~\cite{jain2024livecodebench} provides holistic and contamination-free evaluation of coding capabilities of LLMs. We test models on code generation tasks where models are given a problem statement, which includes a natural language description and example tests (input-output pairs), and is tasked with generating a correct solution.
    \item \textbf{AIME}~\cite{aime26} is a set of math problems from the American Invitational Mathematics Examination.
    \item \textbf{GPQA-Diamond}~\cite{rein2023gpqagraduatelevelgoogleproofqa} is a comprehensive benchmark that evaluates graduate-level knowledge and reasoning capabilities across 285 disciplines.

\end{itemize}
% The following are the details of cross-tokenizer baselines:
% \begin{itemize}
%     \item \textbf{ALM}~\cite{minixhofer2025universalcrosstokenizerdistillationapproximate} uses surface-form blocks as bridging units between the teacher and student, passing teacher likelihoods to the student at the block level.
%     \item \textbf{CDM}~\cite{Li_2024} employs a context-dependent dynamic mapping strategy to construct finer-grained likelihood matching between heterogeneous tokenizers.
% \end{itemize}

\section{Proof}\label{app:proof}

\begin{theorem}[Completeness and Minimality of DPCA]
\label{thm:minimal_partition}
The DPCA algorithm ensures: 1.\textbf{Completeness}: It identifies the entire set of synchronization points between $y^S$ and $y^T$. 2.\textbf{Minimality}: The resulting partition $P^*$ consists of minimal synchronized chunks, such that no chunk within $P^*$ can be further subdivided into smaller synchronized sub-sequences.
\end{theorem}

The \textit{completeness} guarantees that the algorithm never skips a synchronization point, regardless of the extent of the vocabulary mismatch.  The \textit{minimality} implies that each identified chunk is a minimal closure. Consequently, the resulting partition has the maximal possible number of chunks. This is crucial for OPD as it ensures that the teacher's log-probabilities are localized to the smallest possible token clusters, which avoids the ambiguous signal over excessively large chunks.

\begin{proof}
We prove this by contradiction and the property of greedy first-match.

\textit{Completeness}: Suppose there exists a synchronization point $(i^*, j^*) \in \mathcal{V}_{\text{sync}}$ that the algorithm fails to identify. Let $(i_r, j_r)$ be the last synchronization point identified by the algorithm such that $i_r < i^*$. During the subsequent search, the DPCA pointers advance based on the catch-up heuristic. Due to the strict monotonicity $|D(\bm{\tau} \odot \tau')| > |D(\bm{\tau})|$, the pointers will traverse the state space and eventually reach the indices $(i^*, j^*)$. At this state, the condition $D(y^S_{i_r:i^*}) = D(y^T_{j_r:j^*})$ is satisfied. Since the algorithm terminates the current chunk at the \textit{first} instance this equality holds, it must identify $(i^*, j^*)$ as a synchronization point, contradicting the assumption.

\textit{Minimality}: Since DPCA identifies every synchronization point in $\mathcal{V}_{\text{sync}}$ and partitions the sequences at every such point, any resulting chunk between $(i_k, j_k)$ and $(i_{k+1}, j_{k+1})$ contains no intermediate synchronization points by definition. Thus, each chunk is an atomic (minimal) synchronized unit.
\end{proof}

\end{document}